\documentclass[letterpaper, 10 pt, conference]{ieeeconf}  

\IEEEoverridecommandlockouts                              

\usepackage{graphics} 
\usepackage{epsfig} 
\usepackage{mathptmx} 
\usepackage{amsmath} 
\usepackage{amssymb}  
\usepackage{amsfonts}

\usepackage{amsthm}
\usepackage{graphicx}
\usepackage{makecell}
\usepackage{multirow}
\usepackage{algorithm}
\usepackage{algpseudocode}
\usepackage{bm}
\usepackage{upgreek}
\usepackage{subcaption}
\usepackage{hyperref}
\usepackage{geometry}
\newtheorem{assumption}{Assumption}
\newtheorem{proposition}{Proposition}

\title{\LARGE \bf
Generalized Maximum Likelihood Estimation for Perspective-n-Point Problem
}

\author{Tian Zhan, Chunfeng Xu, Cheng Zhang* and Ke Zhu
\thanks{All authors are with the School of Aerospace Engineering, Beijing Institute of Technology, Beijing, China. {\tt\small $\{$zhantian, cfxu, zhangcheng, kezhu$\}$@bit.edu.cn}}%
\thanks{*Corresponding author.}
\thanks{Open source code can be found at \url{https://github.com/zhantian00/gmlpnp}}
}

\begin{document}

\maketitle
\thispagestyle{empty}
\pagestyle{empty}

\begin{abstract}
    The Perspective-n-Point (PnP) problem has been widely studied in the literature and applied in various vision-based pose estimation scenarios. However, existing methods ignore the anisotropy uncertainty of observations, as demonstrated in several real-world datasets in this paper. This oversight may lead to suboptimal and inaccurate estimation, particularly in the presence of noisy observations. To this end, we propose a generalized maximum likelihood PnP solver, named GMLPnP, that minimizes the determinant criterion by iterating the GLS procedure to estimate the pose and uncertainty simultaneously. Further, the proposed method is decoupled from the camera model. Results of synthetic and real experiments show that our method achieves better accuracy in common pose estimation scenarios, GMLPnP improves rotation/translation accuracy by $4.7\%$/$2.0\%$ on TUM-RGBD and $18.6\%$/$18.4\%$ on KITTI-360 dataset compared to the best baseline. It is more accurate under very noisy observations in a vision-based UAV localization task, outperforming the best baseline by $34.4\%$ in translation estimation accuracy.
\end{abstract}


\section{INTRODUCTION}
\label{sec:intro}

Perspective-n-Point (PnP) is a classic robotics and computer vision problem that aims to recover the 6-DoF pose given a set of $n$ 3D object points in the world frame and the corresponding 2D image projection points on a calibrated camera. It is critical in various vision and robotics applications, e.g., vision-based localization~\cite{ORBSLAM3_TRO, vinsmono}, 3D reconstruction~\cite{schoenberger2016sfm}, etc.

In the literature, most PnP solvers ignore \textbf{the anisotropy of observation noise}. In the context of vision-based pose estimation with respect to sparse features, existing works assume that the observation of object points and their projection on the image is accurate or with an isotropic Gaussian noise, i.e., with the covariance formed $\sigma^2\mathbf{I}$~\cite{cpnp, epnpu}. However, this may not hold in real-world data, as the observation of object points is derived from different sensors' measurements and techniques, along with the propagation of image point noise, resulting in anisotropic uncertainty. In general cases, the distribution of the noise may not be known in advance, thus how to estimate the observation uncertainty is an essential issue. Furthermore, there is a practical need for the generalization of PnP methods to cope with omnidirectional camera models. Most existing works build upon the perspective camera model (e.g., pinhole camera)~\cite{epnp, epnpu, sqpnp}, while the omnidirectional camera (e.g., fisheye camera) is often used in vision-based localization. Coupling the solver with the camera model restricts the application of these methods.

In this paper, we propose a generalized maximum likelihood PnP (GMLPnP) solver that considers the anisotropy of observation uncertainty. The term \textit{generalized} comes from two aspects: (1) generalized least squares and (2) generalization for the camera model. The contributions of this paper are:
\begin{enumerate}
    \item Show that many real-world data have the property of anisotropic uncertainty.
    \item Devise a novel PnP solver featuring:
        \begin{itemize}
            \item Its solution is statically optimal in the sense of maximum likelihood.
            \item It simultaneously estimates the distribution parameter of the observation uncertainty by iterated generalized least squares (GLS) procedure.
            \item The estimation is consistent, i.e., convergent in probability.
            \item The proposed PnP solver is decoupled from the camera model.
        \end{itemize}
    \item The proposed method is evaluated by experiments including synthetic and real-world data. An application in UAV localization by vision is shown in section \ref{sec:uav}, which is the original motivation behind our work.
\end{enumerate}

\section{RELATED WORKS}
In the past several decades, researchers have dedicated themselves to finding the optimal and more efficient solution for the PnP problem. Algorithms that depend on a fixed number of points~\cite{kneip2011p3p,bujnak2008p4p} are practically sub-optimal since they do not make full use of the information of all the observed points, and their stability under noisy measurements is limited. Among the non-iterative methods, the most well-known efficient solution is the EPnP~\cite{epnp}, which solves the least squares formulation based on principal component analysis. In DLS~\cite{dls}, a nonlinear object space error is minimized by the least squares.

Iterative methods usually provide better precision while yielding more computational cost. Classic Gauss-Newton refinement~\cite{hartley2003multiple}, or motion-only bundle adjustment (BA) in some literature~\cite{mur2015orbslam}, minimizes the reprojection error defined in the image plane and is often minimized on the manifold of $\mathrm{SO}(3)$ or $\mathrm{SE}(3)$, forming an unconstrant non-linear optimization problem~\cite{forstner2010minimal}. REPPnP~\cite{reppnp} includes algebraic outlier rejection that removes sequentially eliminating outliers exceeding some threshold. PPnP~\cite{ppnp} formulates an anisotropic orthogonal Procrustes problem. The error between the object and the reconstructed image points is minimized in a block relaxation~\cite{de1994block} scheme. SQPnP~\cite{sqpnp} obtains the global minimum from the regional minimum computed with a sequential quadratic programming scheme starting from several initials. CPnP~\cite{cpnp}, analyzes and subtracts the asymptotic bias of a closed-form least squares solution, resulting in a consistent estimate. A very recent EPnP-based work ACEPnP~\cite{sun2024efficient} integrates geometry constraints into control points formulation and reformulates the LS to quadratic constraints quadratic programming. In general, iteration methods yield better accuracy, but they require a good initialization to avoid trapping in a local minima. Hence, the combination of non-iterative initialization and iterative refinement is commonly used by the methods mentioned above.

Many solutions are only geometrically optimal since they do not consider the uncertainties of the observations. Methods that involve observation uncertainty have been explored by researchers recently and are most relevant to our work. MLPnP~\cite{urban2016mlpnp} is a maximum likelihood solution to the PnP problem incorporating image observation uncertainty, thus statistically optimal. CEPPnP~\cite{ceppnp} formulates the PnP problem as a maximum likelihood minimization approximated by an unconstrained Sampson error function that penalizes the noisy correspondences. EPnPU and DLSU~\cite{epnpu} are uncertainty-aware pose estimation methods based on EPnP and DLS, a modified motion-only BA is introduced to take 3D and 2D uncertainties into account. In the above methods that incorporate observation uncertainty, noise distribution is acquired by feature detector, triangulation sequence of images or prior knowledge of sensors. We argue that these prior may not be available in real-world applications, but the uncertainty can be estimated directly from the residuals of the error function.


In addition, formulation decoupled from the camera model is essential in many modern applications incorporating various camera models, e.g., the fisheye camera is widely used in visual SLAM~\cite{ORBSLAM3_TRO, vinsmono}. UPnP~\cite{upnp} is a linear, non-iterative unified method that is decoupled from the camera model and extends the solution to the NPnP (Non-Perspective-n-Point) problem. It solves the problem by a closed-form computation of all stationary points of the sum of squared object space errors. The DLS, MLPnP, and gOp~\cite{shree2001general} also include projection rays and formulate errors in object space, which can be solved by providing the unprojection function passing from pixels to projection rays. 

\section{Problem Formulation}
Given a set of observed object points $\mathbf{p}_i\in\mathbb{R}^3,i=1,\cdots,n$ and their corresponding projection $\mathbf{u}_i$ on image plane of an unknown camera frame, PnP is aimed to recover the motion, i.e., rotation $\mathbf{R}\in \mathrm{SO}(3)$ and translation $\mathbf{t}\in\mathbb{R}^3$. This paper considers general cases when the number of points $n$ is from dozens to hundreds in common vision-based localization tasks. Existing methods primarily aim to minimize the reprojection error defined on the normalized image plane, known as the \textit{gold standard}~\cite{hartley2003multiple}. This method is restricted to a perspective camera, whereas we are seeking a formulation that is decoupled from the camera model.

Considering the formulation in object space, define $\mathbf{m}_i\in\mathbb{R}^3$ as the unit projection ray with $\Vert\mathbf{m}_i\Vert=1$, and $s_i$ as the scale factor (depth) of the corresponding point. The geometry relation is given by
\begin{equation}
    \mathbf{p}_i = s_i\mathbf{R}\mathbf{m}_i + \mathbf{t} + \bm{\epsilon}_i,~i=1,\cdots,n
    \label{equ:model}
\end{equation}
where $\bm{\epsilon}_i$ is the disturbance term. The unit projection ray $\mathbf{m}_i = \bm{\pi}^{-1}(\mathbf{u}_i) / \Vert\bm{\pi}^{-1}(\mathbf{u}_i)\Vert$ is obtained by the image point $\mathbf{u}_i$ and inverse projection function $\bm{\pi}^{-1}$ of the camera, e.g., for pinhole camera $\mathbf{m}_i = \mathbf{K}^{-1}\mathbf{u}_i / \Vert\mathbf{K}^{-1}\mathbf{u}_i\Vert$ where $\mathbf{K}$ is the intrinsic matrix. This is a similar formulation used in~\cite{urban2016mlpnp, dls, schweighofer2008globally}, shown in Fig. \ref{fig:vis_pnp}.

\begin{figure}
    \centering
    \includegraphics[width=\linewidth]{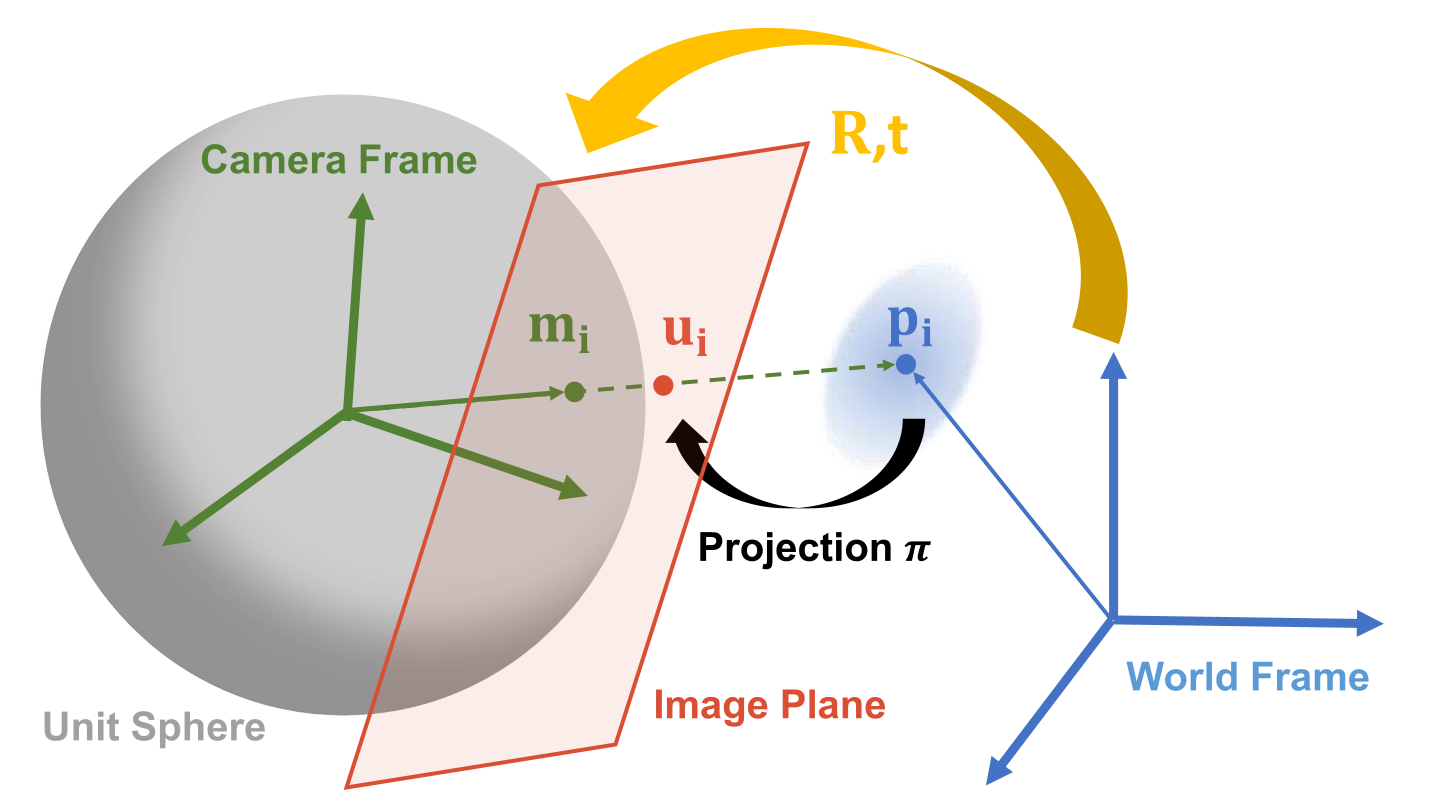}
    \caption{We formulate the model in object space with projection rays, which can cope with perspective and omnidirectional camera models. The blue ellipse cloud visualizes the uncertainty.}
    \label{fig:vis_pnp}
\end{figure}

For the noise $\bm{\epsilon}_i$, we have the following assumption:
\begin{assumption}
    The observation is corrupted by zero mean Gaussian noise $\bm{\epsilon}_i\sim\mathcal{N}(\mathbf{0}, \Sigma)$, and all observation are i.i.d. with covariance $\Sigma$ positive-definite.
    \label{asu:consistency1}
\end{assumption} 
In the literature~\cite{cpnp, epnpu}, the covariance matrix is often considered isotropic, i.e., $\Sigma=\sigma^2\mathbf{I}$, where $\mathbf{I}$ is the identity matrix and $\sigma$ is the standard deviation. We instead relax this assumption and argue the covariance can be anisotropic as in Assumption \ref{asu:consistency1}. The uncertainty of image points can be propagated into the object space~\cite{forstner2010minimal, urban2016mlpnp, epnpu}. The resulting random noise, represented by $\bm{\epsilon}_i$ in \eqref{equ:model}, is modelled by Assumption \ref{asu:consistency1}. It is reasonable to assume $\Sigma$ to be positive-definite because the covariance matrix is positive semidefinite and when $\vert\Sigma\vert=0$, the space spanned by observation noise collapses into a plane or a line, which is unlikely to happen.

\section{Method}
In this section, we first consider the maximum likelihood estimation of the PnP problem in object space when the parameter of noise distribution is known. Then, the method is generalized to simultaneously estimate the pose and infer the covariance of the noise distribution without prior knowledge, named GMLPnP. Finally, we discuss the consistency and convergence of the GMLPnP solver.

\subsection{Maximum Likelihood Estimation with Uncertainty Prior}
Given the model of \eqref{equ:model} and Assumption \ref{asu:consistency1}, we relax the scale constraint, treating each $s_i$ as a free parameter as in~\cite{dls}. When the noise uncertainty $\Sigma$ is known, we introduce maximum likelihood estimation as
\begin{proposition}
The maximum likelihood estimation of transformation $\mathbf{R,t}$ in object space is given by minimizing the error function
\begin{equation}
    \mathcal{E}^2 = \frac{1}{2}\sum_{i=1}^n \left\lVert\mathbf{p}_i - (s_i\mathbf{R}\mathbf{m}_i + \mathbf{t}) \right\rVert^2_\Sigma
    \label{equ:our_cost_func}
\end{equation}
where $\lVert\cdot\rVert_{\Sigma}$ is the Mahalanobis norm, and $\Sigma$ is the covariance matrix of the known noise distribution.
\end{proposition}

\begin{proof}
Denote observations $\mathbf{Y} = \{\mathbf{p}_1,\cdots,\mathbf{p}_n\}$, parameters $\bm{\theta} = \{\mathbf{R,t}, s_1,\cdots, s_n\}$, and residual $\mathbf{e}_i = \mathbf{p}_i - (s_i\mathbf{R}\mathbf{m}_i + \mathbf{t})$, the joint distribution of $\mathbf{Y}$ is
\begin{equation}
    P(\mathbf{Y}|\bm{\theta}, \Sigma) = \prod_{i=1}^n \frac{1}{\sqrt{(2\pi)^3|\Sigma|}} \exp\left(-\frac{1}{2}\mathbf{e}_i^\top\Sigma^{-1}\mathbf{e}_i\right).
\end{equation}

With covariance $\Sigma$ known, the log-likelihood is given by
\begin{equation}
\begin{aligned}
    L(\bm{\theta}) &= \log\prod_{i=1}^n \frac{1}{\sqrt{(2\pi)^3|\Sigma|}} \exp\left(-\frac{1}{2}\mathbf{e}_i^\top\Sigma^{-1}\mathbf{e}_i\right) \\
    &= -\frac{n}{2}\log\left((2\pi)^3 |\Sigma|\right) + \sum_{i=1}^n -\frac{1}{2} \mathbf{e}_i^\top\Sigma^{-1}\mathbf{e}_i \\
    &\propto \sum_{i=1}^n -\frac{1}{2} \left(\mathbf{p}_i - (s_i\mathbf{R}\mathbf{m}_i + \mathbf{t})\right)^\top \Sigma^{-1} \left(\mathbf{p}_i - (s_i\mathbf{R}\mathbf{m}_i + \mathbf{t})\right) \\
    &= - \frac{1}{2}\sum_{i=1}^n \left\lVert\mathbf{p}_i - (s_i\mathbf{R}\mathbf{m}_i + \mathbf{t}) \right\rVert^2_\Sigma.
    \label{equ:loglikehood1}
\end{aligned}
\end{equation}
Thus minimizing \eqref{equ:our_cost_func} is equivalent to maximizing the log likelihood.
\end{proof}

The parameters in \eqref{equ:our_cost_func} are naturally divided into two blocks, namely the pose parameters $\mathbf{R},\mathbf{t}$ and the scale parameters $s_1,\cdots, s_n$. We can optimize these two blocks by block relaxation~\cite{de1994block}, where each group of variables is alternatively estimated while keeping others fixed. The rotation is often minimized on the manifold of $\mathrm{SO}(3)$~\cite{forstner2010minimal}, thus formulating an unconstrained nonlinear optimization problem. The Jacobians of $\mathcal{E}^2$ with respective to $\mathbf{R}\in\mathrm{SO}(3)$ and $\mathbf{t}\in\mathbb{R}^3$ is given by
\begin{equation}
    \frac{\partial\mathcal{E}^2}{\partial{\mathbf{R}}} = \sum_{i=1}^n \lfloor s_i\mathbf{R}\mathbf{m}_i\rfloor_\times\Sigma^{-1}\mathbf{e}_i, ~
    \frac{\partial\mathcal{E}^2}{\partial{\mathbf{t}}} = \sum_{i=1}^n -\Sigma^{-1}\mathbf{e}_i
\end{equation}
where $\lfloor\cdot\rfloor_\times$ is to take the skew-symmetric matrix of a vector.

For the optimization of scale, by setting the partial derivation of $\mathcal{E}^2$ w.r.t. $s_i$ to zero, we obtain
\begin{equation}
    s_i = \frac{(\mathbf{p}_i - \mathbf{t})^\top\Sigma^{-1}\mathbf{R}\mathbf{m}_i}{(\mathbf{R}\mathbf{m}_i)^\top\Sigma^{-1}\mathbf{R}\mathbf{m}_i}.
    \label{equ:scale}
\end{equation}

\subsection{Generalized Maximum Likelihood Estimation}

\begin{algorithm}[tb]
\centering
\caption{Generalized Maximum Likelihood PnP estimator}\label{alg:gmlpnp}
\begin{algorithmic}
    \Require \\
    Object points $\{\mathbf{p}_i\in\mathbb{R}^3\}_{i=1}^n$ \\
    Projection ray $\{\mathbf{m_i}\in\mathbb{R}^3\}_{i=1}^n$ \\

    \State Estimate an initial guess of $\hat{\mathbf{R}}^{(0)}, \hat{\mathbf{t}}^{(0)}$.
    \State $k\leftarrow 1$
    \Repeat
        \State $\hat{\Sigma}^{(k-1)} \leftarrow \frac{1}{n} \sum_{i=1}^n \mathbf{e}_i^{(k-1)}{\mathbf{e}_i^{(k-1)}}^\top$
        \For {$i \leftarrow 1 $ to $ n$}
        \State $$\hat{s}_i^{(k-1)}\leftarrow \frac{(\mathbf{p}_i - \hat{\mathbf{t}}^{(k-1)})(\hat{\Sigma}^{(k-1)})^{-1}\hat{\mathbf{R}}^{(k-1)}\mathbf{m}_i}{(\hat{\mathbf{R}}^{(k-1)}\mathbf{m}_i)^\top(\hat{\Sigma}^{(k-1)})^{-1}\hat{\mathbf{R}}^{(k-1)}\mathbf{m}_i}$$ 
        \EndFor
        \State $\hat{\mathbf{R}}^{(k)}, \hat{\mathbf{t}}^{(k)} \leftarrow \mathop{\arg\min}\limits_{\mathbf{R}, \mathbf{t}}\sum_{i=1}^n \left\lVert\mathbf{p}_i - (\hat{s}_i^{(k-1)}\mathbf{R}\mathbf{m}_i + \mathbf{t}) \right\rVert^2_{\hat{\Sigma}^{(k-1)}}$
        \For {$i \leftarrow 1 $ to $ n$}
        \State$\mathbf{e}_i^{(k-1)} \leftarrow \mathbf{p}_i - (\hat{s}_i^{(k-1)}\hat{\mathbf{R}}^{(k-1)}\mathbf{m}_i + \hat{\mathbf{t}}^{(k-1)})$
        \EndFor
        \State $k \leftarrow k + 1$
    \Until $|\hat\Sigma^{(k)} - \hat\Sigma^{(k-1)}| < \text{Threshold}$ 
    \State \Return $\hat{\mathbf{R}}, \hat{\mathbf{t}}$
\end{algorithmic}
\end{algorithm}

In more general cases, the uncertainty of the noise is unknown. Hence we develop a PnP solver by generalized nonlinear least squares estimator, which simultaneously estimates the pose and the covariance. We consider the case when there is a rough initial hypothesis of rotation and translation. In this section, we brought the idea introduced by~\cite{zellner1962, malinvaud1980statistical} to solve the regression for multiresponse data under the assumptions that the disturbance terms in different observations are uncorrelated, but the disturbance terms for different responses in the same observation have a fixed unknown covariance. This is called the \textit{iterated GLS procedure} in~\cite{seber2003nonlinear}.

Treating the covariance matrix as an unknown parameter, the log-likelihood of \eqref{equ:loglikehood1} becomes
\begin{equation}
\begin{aligned}
    L(\bm{\theta}, \Sigma) &\propto \frac{n}{2}\log|\Sigma^{-1}| + \sum_{i=1}^n -\frac{1}{2} \mathbf{e}_i^\top\Sigma^{-1}\mathbf{e}_i \\
    &\propto \log|\Sigma| + \frac{1}{n} \sum_{i=1}^n \mathbf{e}_i^\top\Sigma^{-1}\mathbf{e}_i
    \label{equ:generalized_loss}
\end{aligned}
\end{equation}

The minimization of \eqref{equ:generalized_loss} can be finished in two steps as described in~\cite{seber2003nonlinear}. We first fix $\bm{\theta}$ and minimize \eqref{equ:generalized_loss} w.r.t. to $\Sigma$ by setting the derivative to zero, providing the conditional estimation
\begin{equation}
    \hat{\Sigma}(\bm{\theta}) = \frac{1}{n}\sum_{i=1}^n \mathbf{e}_i\mathbf{e}_i^\top .
    \label{equ:estimate_sigma}
\end{equation}
The resulting value of $\hat{\Sigma}$ is then substituted back into \eqref{equ:generalized_loss}, and the resulting function of $\bm{\theta}$ is minimized with respect to $\bm{\theta}$, which is identical to minimizing \eqref{equ:our_cost_func}. This technique of first eliminating $\Sigma$ to obtain a function of $\bm{\theta}$ is called concentrating the likelihood~\cite{seber2003nonlinear}. We summarize the method in Algorithm \ref{alg:gmlpnp}. In practice, we found that two iterations are usually enough for convergence (as shown in Fig. \ref{fig:det_over_iter}), and the convergence threshold can be set as $1\times 10^{-5}$.

We introduce another interpretation of the minimization of \eqref{equ:generalized_loss}. When substitute \eqref{equ:estimate_sigma} into \eqref{equ:generalized_loss} gives the conditional log-likelihood
\begin{equation}
    L(\bm{\theta}, \hat{\Sigma}(\bm{\theta})) \propto -\log|\mathbf{V}(\bm{\theta})|, \mathrm{~where~} \mathbf{V}(\bm{\theta}) = \sum_{i=1}^n \mathbf{e}_i\mathbf{e}_i^\top
    \label{equ:determinate_criterion}
\end{equation}
which is equivalent to directly minimizing $|\mathbf{V}(\bm{\theta})|$. This is called \textit{the determinant criterion} in~\cite{bates1988nonlinear}. Geometrically, $|\mathbf{V}(\bm{\theta})|$ corresponds to the square of the volume of the parallelepiped spanned by the residual vectors. Minimizing the determinant corresponds to minimizing the volume enclosed by the residual vectors.

\subsection{Discussion on Consistency and Convergence}
To discuss the asymptotic properties of the estimation of GMLPnP,  consider our nonlinear model $\mathbf{f}(\mathbf{m};\bm{\theta}) = s\mathbf{R}\mathbf{m} + \mathbf{t}$. Define set $\mathcal{M}$ as the unit sphere in $\mathbb{R}^3$, so that $\mathbf{m}\in\mathcal{M}$. We rewrite the rotation as Lie algebra $\mathbf{R}=\log \lfloor\bm{\phi}\rfloor_\times$ where $\bm{\phi}\in\mathfrak{so}(3)$. Then, the parameter $\bm{\theta}=[\bm{\phi}^\top~ \mathbf{t}^\top]^\top$ is in a subset of $\mathbb{R}^{6}$ denoted as $\Theta$, note here the scale factor $s_i$ is omitted as it depends on $\bm{\theta}$ by \eqref{equ:scale}. Denote the superscript $(\cdot)^*$ as the true value. The generalized least squares estimator, is the value $\hat{\bm{\theta}}_n(\mathbf{S}_n)$ of $\bm{\theta}$ which minimizes
\begin{equation}
    \sum_{i=1}^n\left(\mathbf{p}_i - \mathbf{f}(\mathbf{m}_i;\bm{\theta})\right)^\top \mathbf{S}_n \left(\mathbf{p}_i - \mathbf{f}(\mathbf{m}_i;\bm{\theta})\right).
\end{equation}
where $\mathbf{S}_n$ is a positive definite matrix, called iterated GLS procedure, where in our case it is the inverse of the covariance matrix \eqref{equ:estimate_sigma}, 

We further make the following assumptions:
\begin{assumption}
    The function $\mathbf{f}(\mathbf{m}; \bm{\theta})$ is continuous on $\mathcal{M}\times\Theta$.
    \label{asu:consistency2}
\end{assumption}
\begin{assumption}
    $\mathcal{M}$ and $\Theta$ are closed, bounded (compact) subset of $\mathbb{R}^3$ and $\mathbb{R}^{6}$ respectively.
    \label{asu:consistency3}
\end{assumption}
Assumption \ref{asu:consistency3} is not a serious restriction, as most parameters are bounded by the physical constraints of the system being modeled~\cite{seber2003nonlinear}. In our case, it is reasonable to bound $\vert\bm{\phi}\vert\leq \pi$ which avoids ambiguities related to multiple angle representations of the same rotation. The translation $\mathbf{t}\in\mathbb{R}^3$ is bounded by the maximum measuring distance of the sensor (camera, LiDAR, etc.).

\begin{assumption}
    The observations $\mathbf{m}$ are such that $H_n(\mathbf{m})\rightarrow H(\mathbf{m})$, where $H_n(\mathbf{m})$ is the empirical distribution function and $H(\mathbf{m})$ is a distribution function.
    \label{asu:consistency4}
\end{assumption}
Assumption \ref{asu:consistency4} indicates that the sample points $\mathbf{m}_1,\cdots,\mathbf{m}_n$ is a random sample from some distribution with distribution function $H(\mathbf{m})$.

\begin{assumption}
    If $\mathbf{f}(\mathbf{m};\bm{\theta}) = \mathbf{f}(\mathbf{m};\bm{\theta}^*)$, then $\bm{\theta}=\bm{\theta}^*$.
    \label{asu:consistency5}
\end{assumption}
Assumption \ref{asu:consistency5} is satisfied under general conditions~\cite{seber2003nonlinear}.

It has been given by Malinvaud et al.~\cite{malinvaud1980statistical} that if Assumptions \ref{asu:consistency1} to \ref{asu:consistency5} hold, then the estimation $\hat{\bm{\theta}}_n(\mathbf{S}_n)$ and
\begin{equation}
    \hat{\Sigma}_n(\mathbf{S}_n) = \sum_{i=1}^n\left(\mathbf{p}_i - \mathbf{f}(\mathbf{m}_i;\hat{\bm{\theta}}_n(\mathbf{S}_n))\right) \left(\mathbf{p}_i - \mathbf{f}(\mathbf{m}_i;\hat{\bm{\theta}}_n(\mathbf{S}_n))\right)^\top
\end{equation}
are consistent estimators of the true value $\bm{\theta}^*$ and $\Sigma^*$ respectively, i.e.,
\begin{equation}
    \mathop{\mathrm{plim}}\limits_{n\rightarrow \infty} \hat{\bm{\theta}}_n = \bm{\theta}^*, ~ \mathop{\mathrm{plim}}\limits_{n\rightarrow \infty} \hat{\Sigma} = \Sigma^*.
\end{equation}

Furthermore, under these assumptions and certain conditions, Phillips et al.~\cite{Phillips1976THEIM} show that the iterated GLS procedure converges for large enough $n$ and the limit point is independent of the starting value of $\Sigma^{(0)}$.

\section{Experiments}
\label{sec:experiments}

We conduct synthetic and real data experiments to evaluate our method, GMLPnP, by comparing with iterative and non-iterative PnP solvers listed in Table~\ref{tab:pnp_algs}. All experiments are conducted via a desktop with Intel Core i7-9700F CPU. The GMLPnP algorithm is implemented in a graph optimization manner with $\mathbf{g^2o}$ \cite{g2o} library, and the initial guess is given by MLPnP~\cite{urban2016mlpnp}.

\begin{table}[htb]
  \caption{The PnP algorithms compared in experiments and their property. (\checkmark) depicts methods that are initialized with a non-iterative method and followed by an iterative refinement. \textit{BA} refers to the motion-only bundle adjustment. \textit{Uncertainty} indicates if the method incorporates uncertainty. \textit{Impl.} indicates the language implemented in our experiment.}
  \label{tab:pnp_algs}
  \centering
  \begin{tabular}{|c|c|c|c|c|}
    \hline
    Method & \thead{Iterative} & \thead{Uncertainty} & \thead{Camera} & Impl. \\
    \hline
    EPnP~\cite{epnp} & & & Perspective & C++ \\
    BA~\cite{hartley2003multiple} & \checkmark & & Perspective & C++ \\
    PPnP~\cite{ppnp} & \checkmark & & Perspective & C++ \\
    SQPnP~\cite{sqpnp} & \checkmark & & Perspective & C++ \\
    CPnP~\cite{cpnp} & (\checkmark) & & Perspective & C++ \\
    UPnP~\cite{upnp} & (\checkmark) & & Decoupled & C++\\
    MLPnP~\cite{urban2016mlpnp} & (\checkmark) & \checkmark & Decoupled & C++ \\
    REPPnP~\cite{reppnp} & \checkmark & & Perspective & MATLAB\\
    EPnPU~\cite{epnpu} & (\checkmark) & \checkmark & Perspective & MATLAB\\
    DLSU~\cite{epnpu} & (\checkmark) & \checkmark & Perspective & MATLAB\\
    GMLPnP(ours) & \checkmark & \checkmark & Decoupled & C++ \\
  \hline
  \end{tabular}
\end{table}

\subsection{Synthetic Experiments}
\label{sec:syn_exp}

\subsubsection{Experiment setup}
In the simulation, we assume a pinhole camera with a focal length of $800$ pixels, resolution $640\times480$ pixels, and principal point in the image center. A point cloud is randomly generated in front of the camera within the range of $[-2,2]\times[-2,2]\times[4,8]$ in the camera frame. The world frame's origin is set at the center of the point cloud under a random rotation to the camera frame. The 3D-2D correspondences are obtained by projecting the point cloud to the image plane. We disturb the observation of each object point with i.i.d. zero-mean Gaussian noise with anisotropic covariance $\Sigma=\mathbf{R}_\mathrm{o}\mathrm{diag}(\sigma^2, \sigma_1^2, \sigma_2^2)\mathbf{R}_\mathrm{o}^\top$, which are composed of randomly generated rotation $\mathbf{R}_\mathrm{o}$ and randomly drawn $\sigma_1, \sigma_2$ within the interval $(0, \sigma)$. The 2D projection noise is generated analogously and added to the image points. We show a reference by bringing a method GMLPnP$^*$ that minimizes~\eqref{equ:our_cost_func} with the true covariance matrix known and fixed. The methods that incorporate uncertainties, namely the MLPnP, EPnPU and DLSU, are fed with ground truth covariance. All the simulations are run 500 times, and the results are taken as the mean value. 


\subsubsection{Metrics}
Absolute rotation error and relative translation error~\cite{ceppnp, reppnp, upnp, cpnp} are used to evaluate estimation accuracy. The absolute rotation error is defined as $\mathrm{e}_{\mathrm{rot}} = \max_{k=1}^3\{\arccos(\mathbf{r}_{k,\mathrm{gt}}^\top \cdot \mathbf{r}_{k,\mathrm{est}}) \times 180/\pi\}$ in degrees, where $\mathbf{r}_{k,\mathrm{gt}}$ and $\mathbf{r}_{k,\mathrm{est}}$ are the $k$-th column of the ground truth rotation matrix $\mathbf{R}_{\mathrm{gt}}$ and PnP estimation $\mathbf{R}_{\mathrm{est}}$. The relative translation error is defined as $\mathrm{e}_{\mathrm{trans}}=\Vert\mathbf{t}_{\mathrm{gt}} - \mathbf{t}_{{\mathrm{est}}}\Vert / \Vert\mathbf{t}_{\mathrm{gt}}\Vert$, where $\mathbf{t}_{\mathrm{gt}}$ and $\mathbf{t}_{{\mathrm{est}}}$ are the ground truth and estimated translation vector. 

\subsubsection{Results}
Fig. \ref{fig:error_vs_num_of_points} shows the pose estimation error with respect to the number of points. The $\sigma$ is set to be $0.1$ meters for object point noise and $1$ pixel for the image point noise. Our methods achieve the best accuracy. The estimation accuracy with respect to noise level is shown in Fig.~\ref{fig:error_vs_noise}, the number of points is set to be $50$, and object point noise $\sigma$ varies from 0.02 to 0.5 meters, the corresponding image point noise varies from 0.2 pixels to 5 pixels accordingly. GMLPnP$^*$ reasonably achieves the highest precision with the prior knowledge of uncertainty. GMLPnP does not lag behind by too much due to the simultaneous estimation of pose and uncertainty and is more accurate than all other methods in both synthetic experiments above. As the noise level increases in Fig.~\ref{fig:error_vs_noise}, the accuracy of our method can still remain relatively high. We compare the computing time of the algorithms implemented in C++, as the MATLAB code is much more inefficient and not comparable, results are shown in Fig.~\ref{fig:time_vs_num_of_points}. The execution time of GMLPnP grows linearly and is fast enough for real-time use.

\begin{figure}[tb]
  \includegraphics[width=\linewidth]{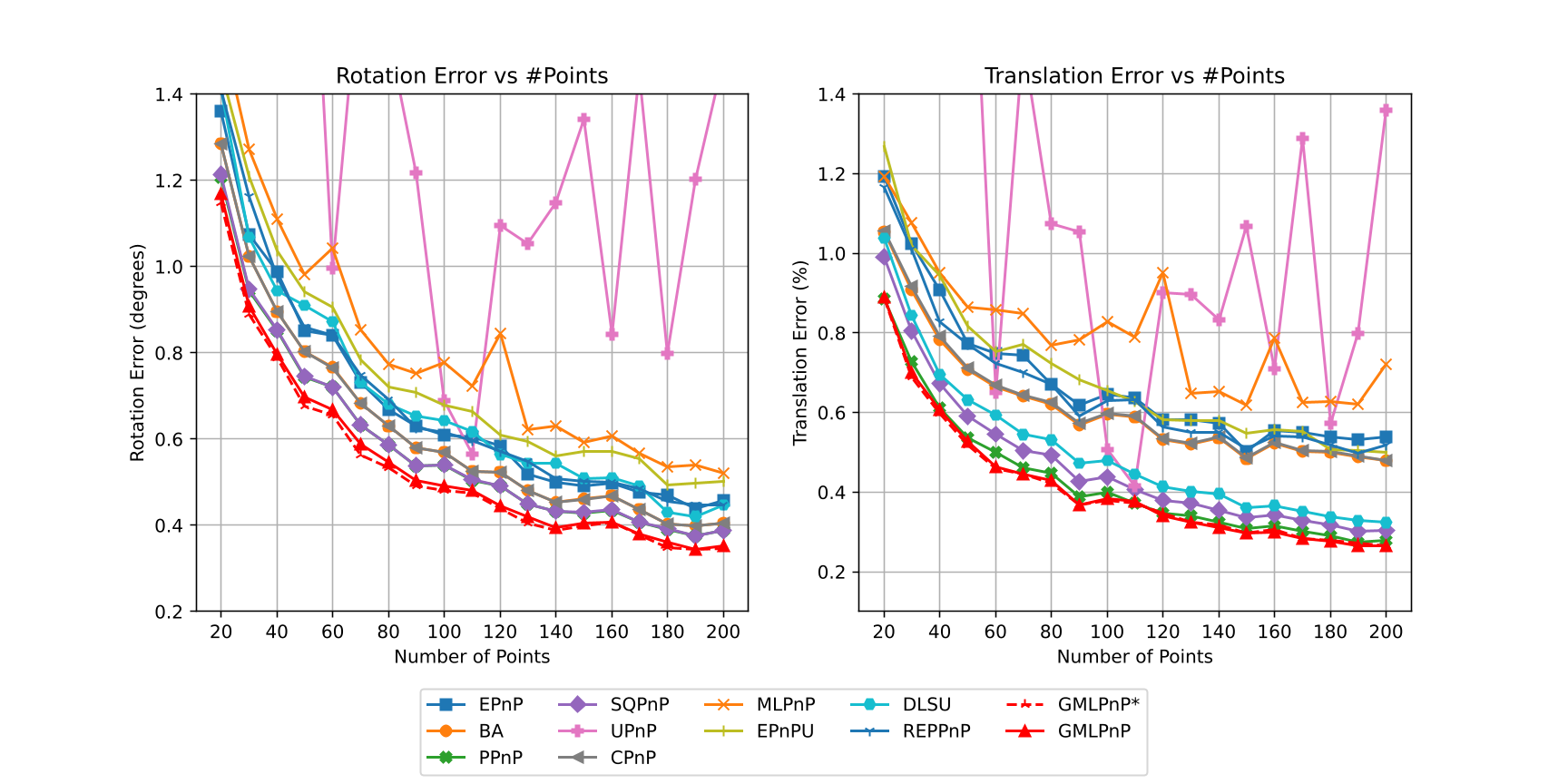}
  \caption{Estimation error vs num points, the number of points is from $20$ to $200$, with object point noise $0.1$ meters and image point noise standard deviation $1$ pixel.}
  \label{fig:error_vs_num_of_points}
\end{figure}

\begin{figure}[tb]
  \includegraphics[width=\linewidth]{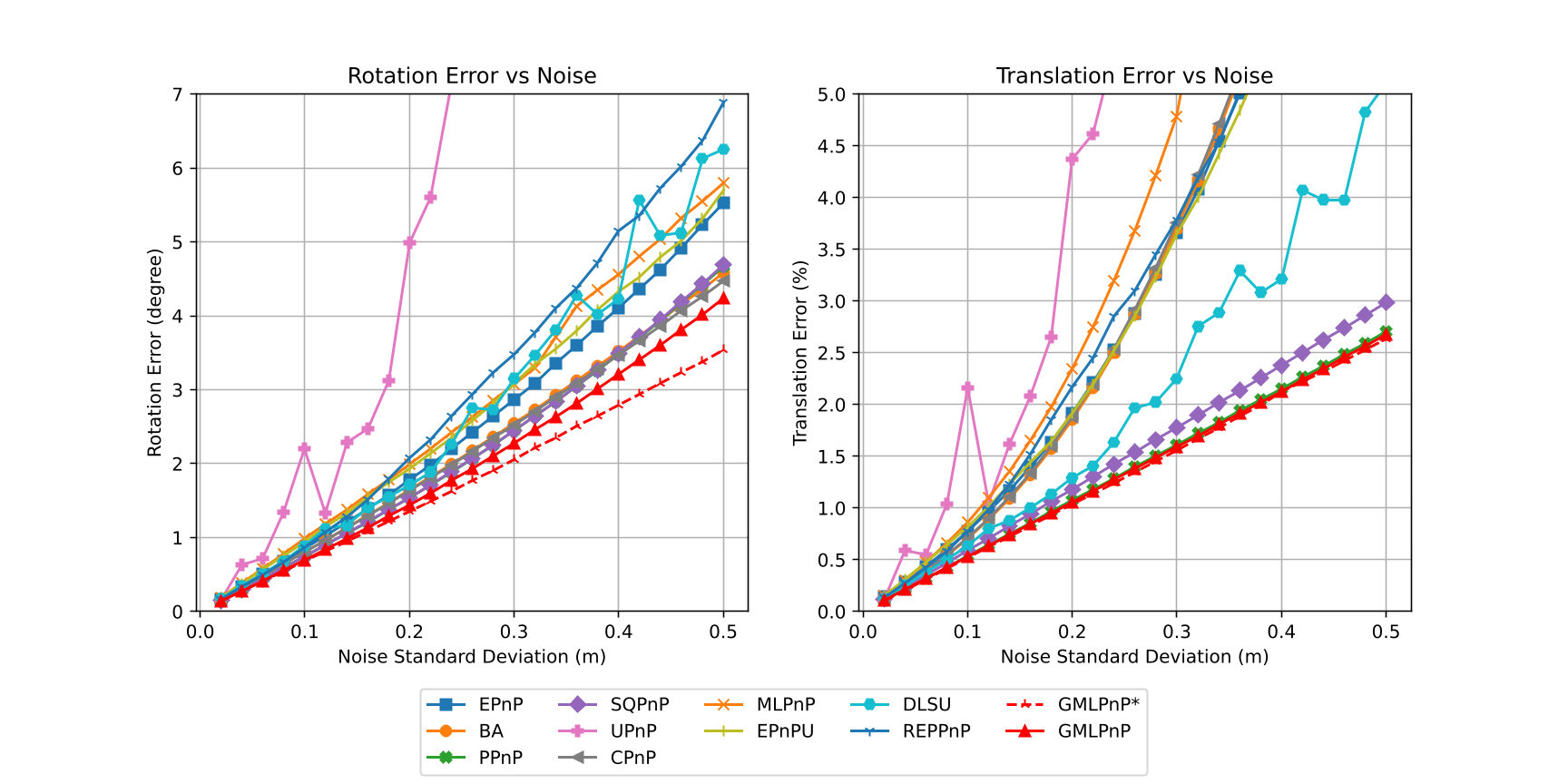}
  \caption{Estimation error vs noise standard deviation, the number of points is set to be $50$, and the object point noise increases from $0.02$ to $0.5$ meters, the corresponding image point noise varies from 0.2 pixels to 5 pixels accordingly.}
  \label{fig:error_vs_noise}
\end{figure}

\subsubsection{Convergence and residual analysis} 
We conduct a convergence and residual analysis experiment setup with 200 points and $\sigma=0.5$ to analyse the uncertainty estimation ability and convergence speed of GMLPnP. These simulations are performed 500 trials, and mean values are reported. Fig.~\ref{fig:det_over_iter} shows the determinate of the estimated covariance, which is minimized iteratively and matches~\eqref{equ:determinate_criterion}. Fig.~\ref{fig:fro_over_iter} shows the Frobenius norm of the difference between the estimated covariance and true covariance, i.e., $\Vert\hat{\Sigma} - \Sigma^*\Vert_F$. These results indicate that our algorithm typically converges in two iterations. Fig.~\ref{fig:synthetic_vis} visualizes a specific case in the synthetic experiment. Fig.~\ref{fig:synthetic_residual} plots the true noise added on the object point observations and the residuals when our algorithm converges, demonstrating that the residuals $\mathbf{e}_1,\cdots,\mathbf{e}_n$ may contain important information about the uncertainties that the previous works ignore. Fig.~\ref{fig:synthetic_cov_mat} shows the true and estimated covariance, which indicates our algorithm can recover the uncertainty well. As an iterative method, we test its robustness to the initialization value in Fig~\ref{fig:error_vs_init}. Results indicate that our method converges consistently under random initial offset and noisy observation.

\begin{figure}[tb]
  \centering
  \includegraphics[width=\linewidth]{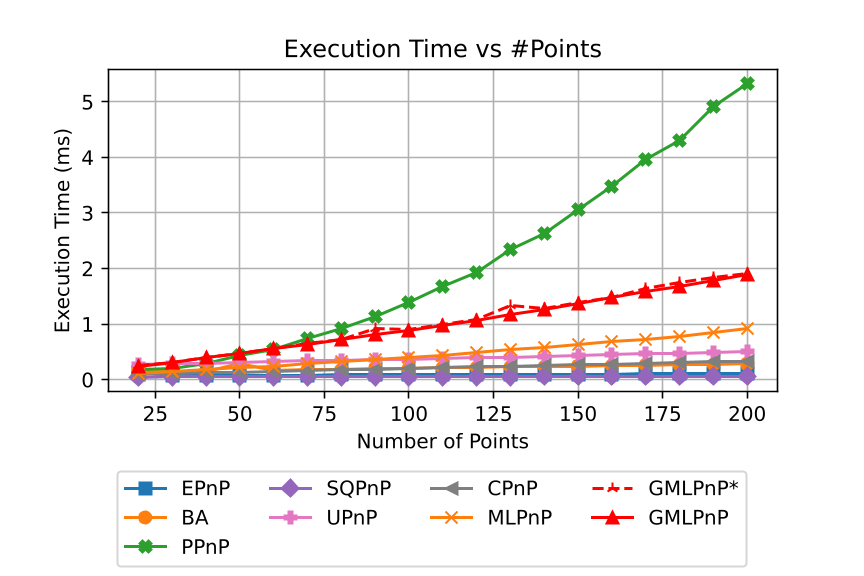}
  \caption{Execution time comparison, methods implemented by C++ are included.}
  \label{fig:time_vs_num_of_points}
\end{figure}

\begin{figure}[tb]
    \centering
    \includegraphics[width=\linewidth]{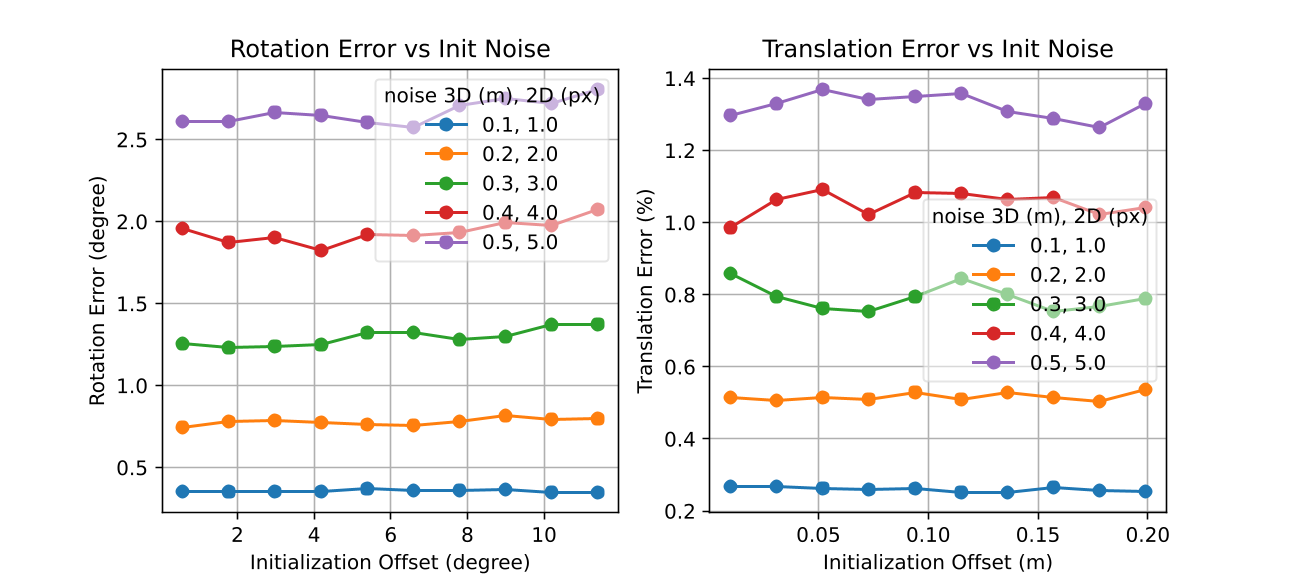}
    \caption{Initialize GMLPnP with ground truth + random offset (rotation and translation are added simultaneously) under different observation noise. The number of points is 200.}
    \label{fig:error_vs_init}
\end{figure}

\begin{figure}[tb]
\centering
\begin{minipage}[t]{0.45\linewidth}
    \centering
    \includegraphics[width=\linewidth]{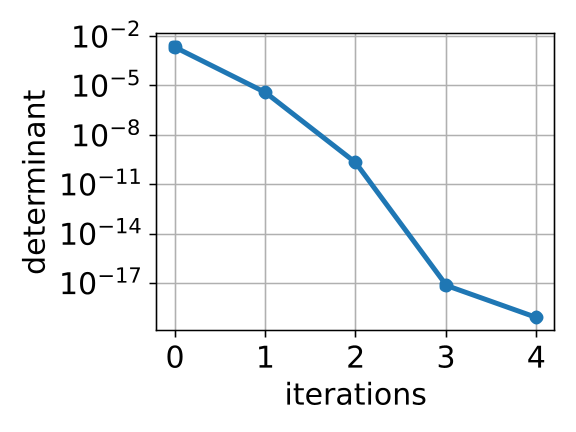}
    \caption{The determinate of estimated covariance is minimized. Iteration $0$ indicates the initial value.}
    \label{fig:det_over_iter}
\end{minipage}
\hspace{0.01\textwidth}
\begin{minipage}[t]{0.45\linewidth}
    \centering
    \includegraphics[width=\linewidth]{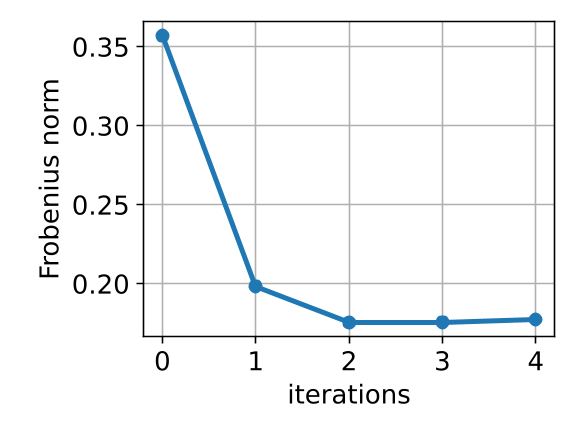}
    \caption{The Frobenius norm of the difference between estimated and true covariance.}
    \label{fig:fro_over_iter}
\end{minipage}
\end{figure}

\begin{figure}[tb]
  \centering
  \begin{subfigure}{\linewidth}
    \includegraphics[width=\linewidth]{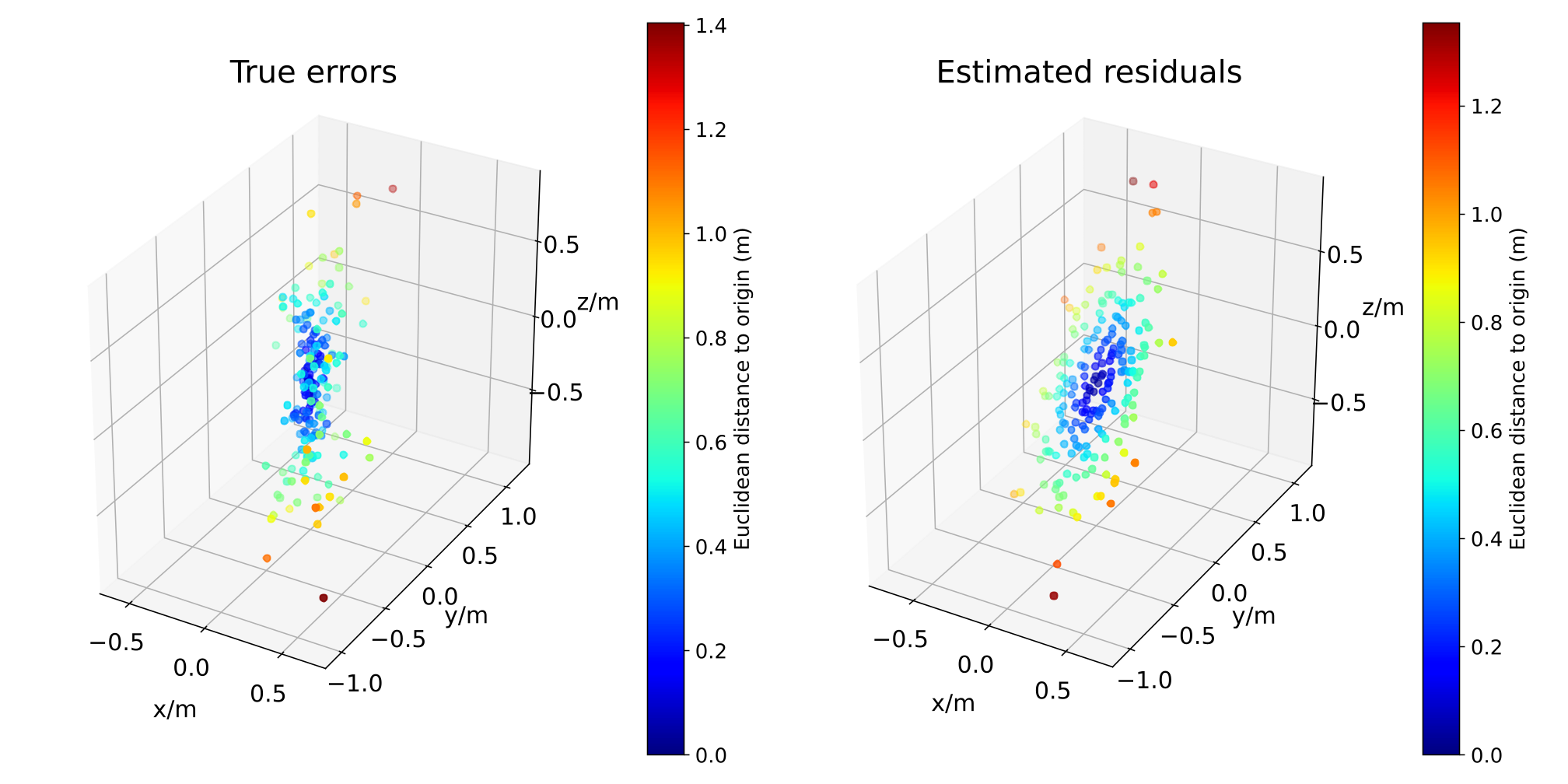}
    \caption{The residual points}
    \label{fig:synthetic_residual}
  \end{subfigure}
  \hfill
  \begin{subfigure}{\linewidth}
    \includegraphics[width=\linewidth]{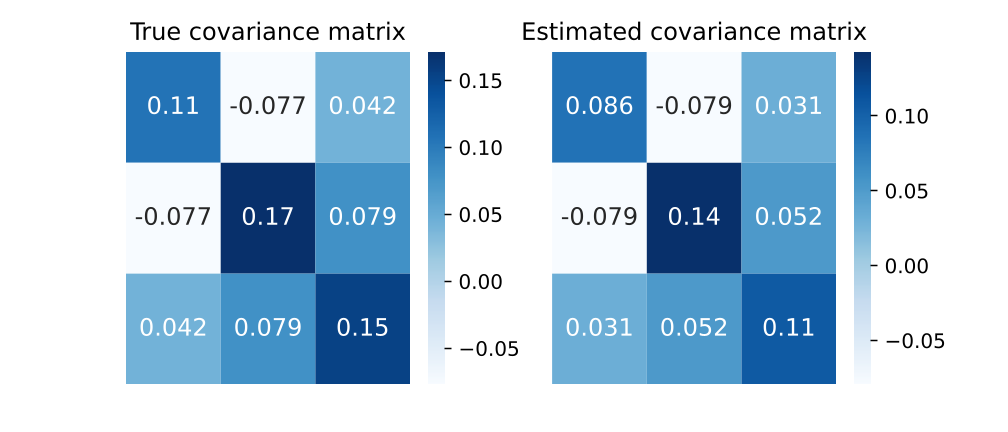}
    \caption{True and estimated covariance}
    \label{fig:synthetic_cov_mat}
  \end{subfigure}
  \caption{Visualization of uncertainty estimation.}
  \label{fig:synthetic_vis}
\end{figure}

\subsection{Motion Estimation with Real Data}
Two real-world datasets, TUM-RGBD \cite{sturm12iros} (with pinhole camera) and KITTI-360 \cite{Liao2022PAMI} (with fisheye camera), are used to evaluate the performance of our proposed method. We report the absolute rotation and translation error. 

\subsubsection{TUM-RGBD dataset} In the TUM-RGBD dataset, the images are recorded by a handheld or robot-mounted RGB-D camera, the RGB camera is a pinhole camera. We sample image pairs with a temporal interval of 0.1 seconds from the \textit{freiburg1} sequences of the dataset. A total of 1662 pairs are sampled. In each pair, the pixel correspondences are registered by ORB descriptor and are brute-force matched with RANSAC outlier rejection. We associate the temporal first image with its corresponding depth map to get the object point observations. Then relative motion between the two frames is estimated by PnP solvers. Note that since we have little prior knowledge about the observation uncertainty, we initiate it as isotropic, i.e., the covariance $\mathbf{I}$. The estimation errors are shown in Table~\ref{tab:tum_kitti}. GMLPnP achieves the best precision, promoting the best baseline by $4.7\%$ in rotation (compared with BA) and $2.0\%$ in translation (compared with SQPnP).

\begin{table}[tb]
  \caption{Motion estimation errors of TUM-RGBD dataset and KITTI-360 dataset with fisheye camera. \textbf{Bold} indicates the best item, \underline{underline} indicates the second best. }
  \label{tab:tum_kitti}
  \centering
\scalebox{1.1}{
  \begin{tabular}{|c|cc|cc|}
    \hline
    \multirow{2}*{~} & \multicolumn{2}{c|}{TUM-RGBD} & \multicolumn{2}{c|}{KITTI-360}\\

     & $\mathbf{R}$($^{\circ}$) & $\mathbf{t}$(cm) & $\mathbf{R}$($^{\circ}$) & $\mathbf{t}$(cm) \\
    \hline
    EPnP       & 1.532 & 1.872 & - & - \\ 
    BA         & \underline{0.915} & 1.728 & - &- \\ 
    PPnP       & 1.060 & 2.334 & - & - \\ 
    SQPnP      & 1.294 & \underline{1.712} & - & - \\ 
    UPnP       & 0.992 & 1.843 & 0.105 & 1.365 \\ 
    CPnP       & 0.974 & 1.849 & - & - \\ 
    MLPnP      & 1.173 & 1.938 & \underline{0.086} & \underline{1.200} \\ 
    EPnPU      & 1.461 & 1.771 & - & -  \\ 
    DLSU       & 0.994 & 1.880 & - & -  \\ 
    REPPnP     & 1.885 & 2.357 & - & -  \\ 
    GMLPnP(ours)     & \textbf{0.872} & \textbf{1.677} & \textbf{0.070} & \textbf{0.979} \\ 
  \hline
  \end{tabular}
}
\end{table}

\subsubsection{KITTI-360 dataset} In the KITTI-360 dataset, a ground vehicle is mounted with omnidirectional fisheye cameras (camera 3, calibrated with MEI model \cite{mei2007single}) along with LiDAR. Analogous to the TUM-RGBD setup, we sample image pairs in a sequence captured by a monocular camera with a temporal interval of 3 consecutive frames. The LiDAR points captured along with the first frame are projected on the corresponding frame and tracked on the second frame with Lucas-Kanade Optical Flow. We sample from sequences 00 to 10, a total of 2138 pairs are used to evaluate. Only methods decoupled from camera model are evaluated. The results are shown in Table~\ref{tab:tum_kitti}. GMLPnP generalizes well on the omnidirectional camera, outperforming the best baseline MLPnP by $18.6\%$ and $18.4\%$ in rotation and translation error respectively.

\subsubsection{Anisotropic uncertainties of real-world data}\label{seq:ani_uncer}
We show the anisotropic uncertainties of real-world data in Fig.~\ref{fig:real_data_residual_analysis} by plotting the residuals and covariance matrix. The object points observed in the first frame are transformed to the second frame by ground truth relative motion to obtain the depth of these points in the second frame. Then, the object points in the second frame are reconstructed by the relative motion, projection rays and the corresponding depths. We get the object point residuals by the observation in the first frame and reconstruction in the second frame. The resulting residuals and covariance show an obvious anisotropic property in these RGB-D and RGB-LiDAR data.

\begin{figure}[tb]
  \centering
  \begin{subfigure}{0.45\linewidth}
    \includegraphics[width=\linewidth]{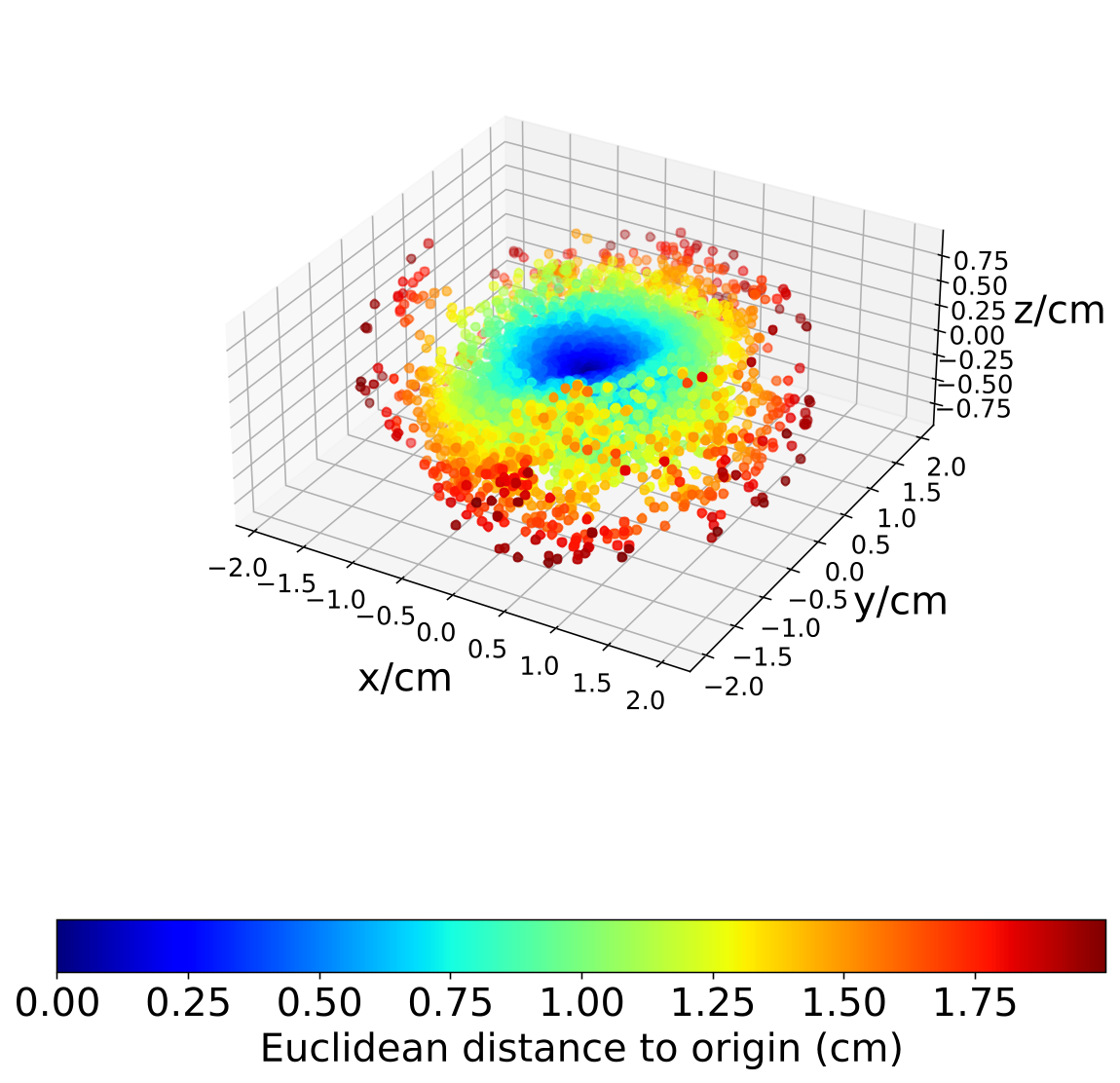}
    \caption{TUM-RGBD residuals}
    \label{fig:tum_residual}
  \end{subfigure}
  \hfill
  \begin{subfigure}{0.45\linewidth}
    \includegraphics[width=\linewidth]{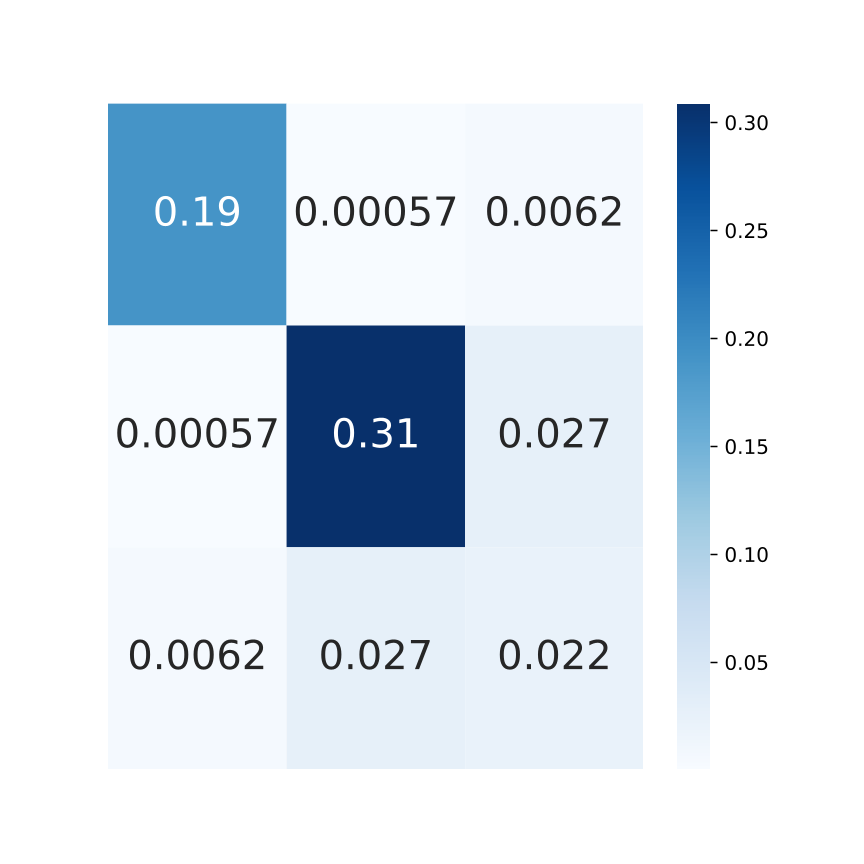}
    \caption{TUM-RGBD covariance}
    \label{fig:tum_cov_mat}
  \end{subfigure}
\hfill
  \begin{subfigure}{0.45\linewidth}
    \includegraphics[width=\linewidth]{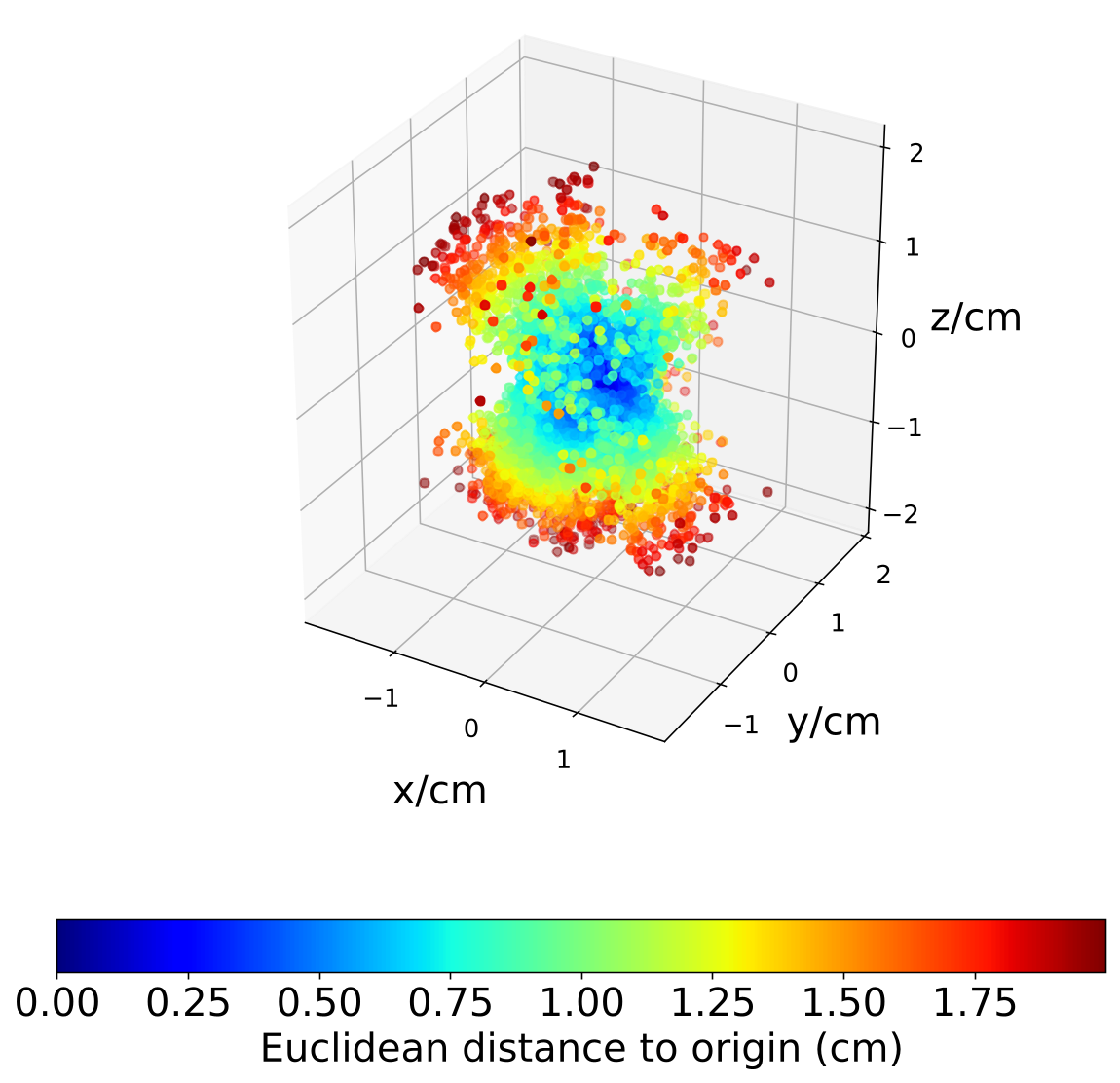}
    \caption{KITTI-360 residuals}
    \label{fig:kitti_residual}
  \end{subfigure}
  \hfill
  \begin{subfigure}{0.45\linewidth}
    \includegraphics[width=\linewidth]{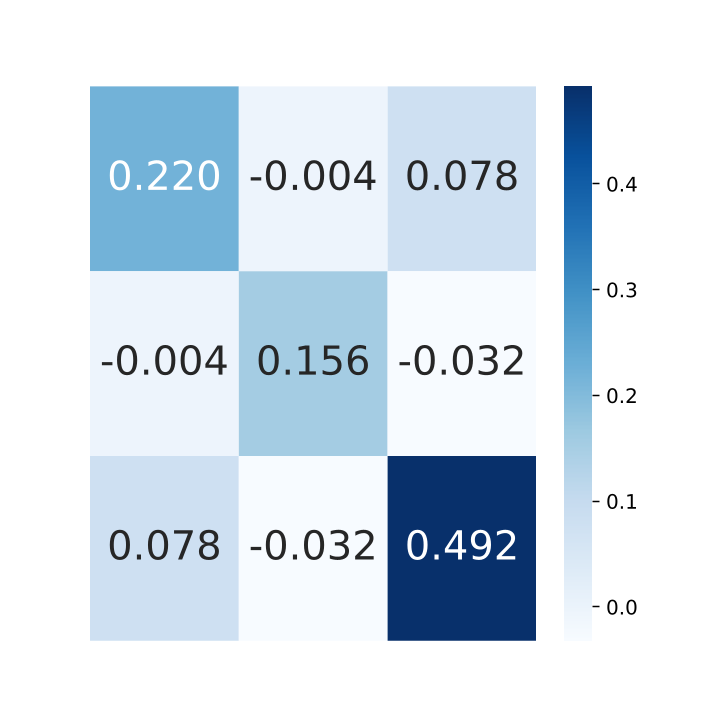}
    \caption{KITTI-360 covariance}
    \label{fig:kitti_cov_mat}
  \end{subfigure}
  \caption{Uncertainty visualization of TUM-RGBD and KITTI-360 datasets show an obvious anisotropic property.}
  \label{fig:real_data_residual_analysis}
\end{figure}

\subsection{Application in Vision-based UAV localization}
\label{sec:uav}

\begin{figure*}[tb]
  \centering
  \begin{subfigure}{0.58\linewidth}
      \includegraphics[width=\linewidth]{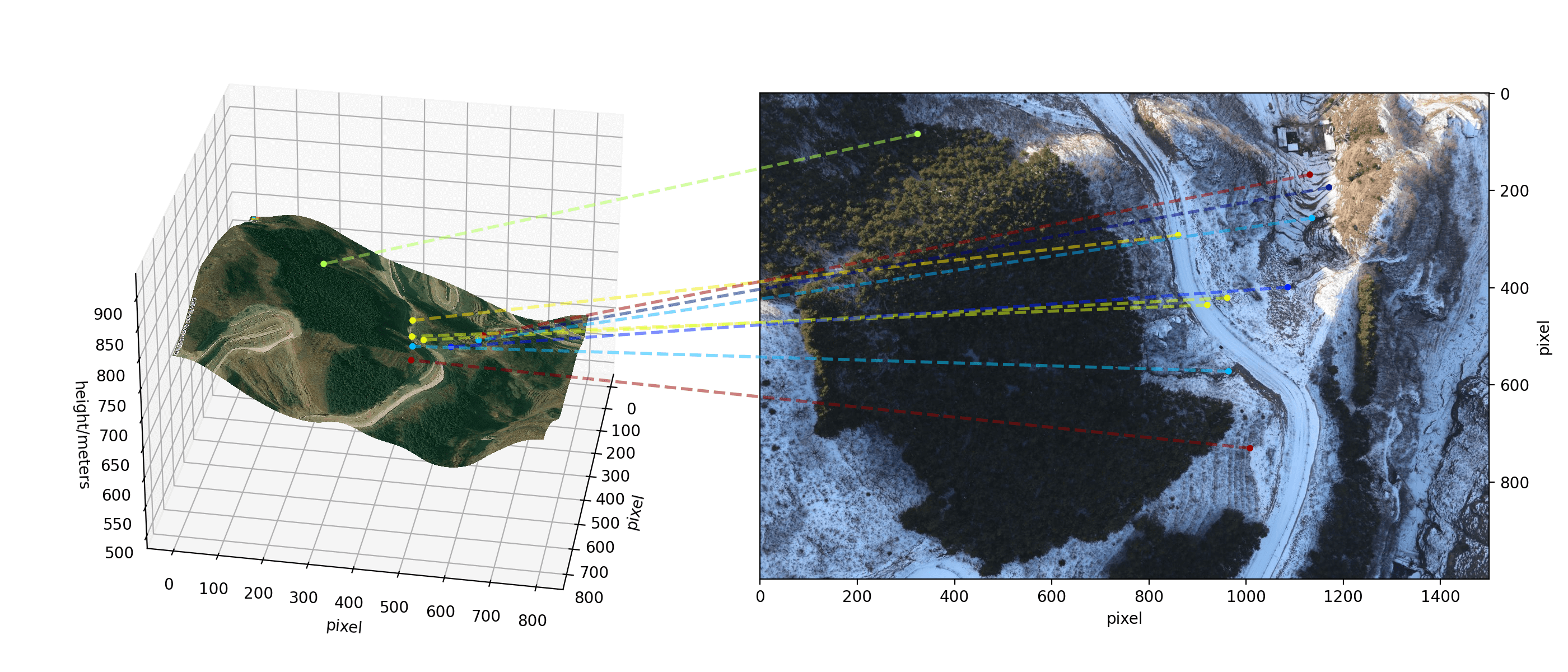}
      \caption{UAV localization as a PnP problem.}
      \label{fig:uav_localization}
  \end{subfigure}
  \begin{subfigure}{0.2\linewidth}
      \includegraphics[width=\linewidth]{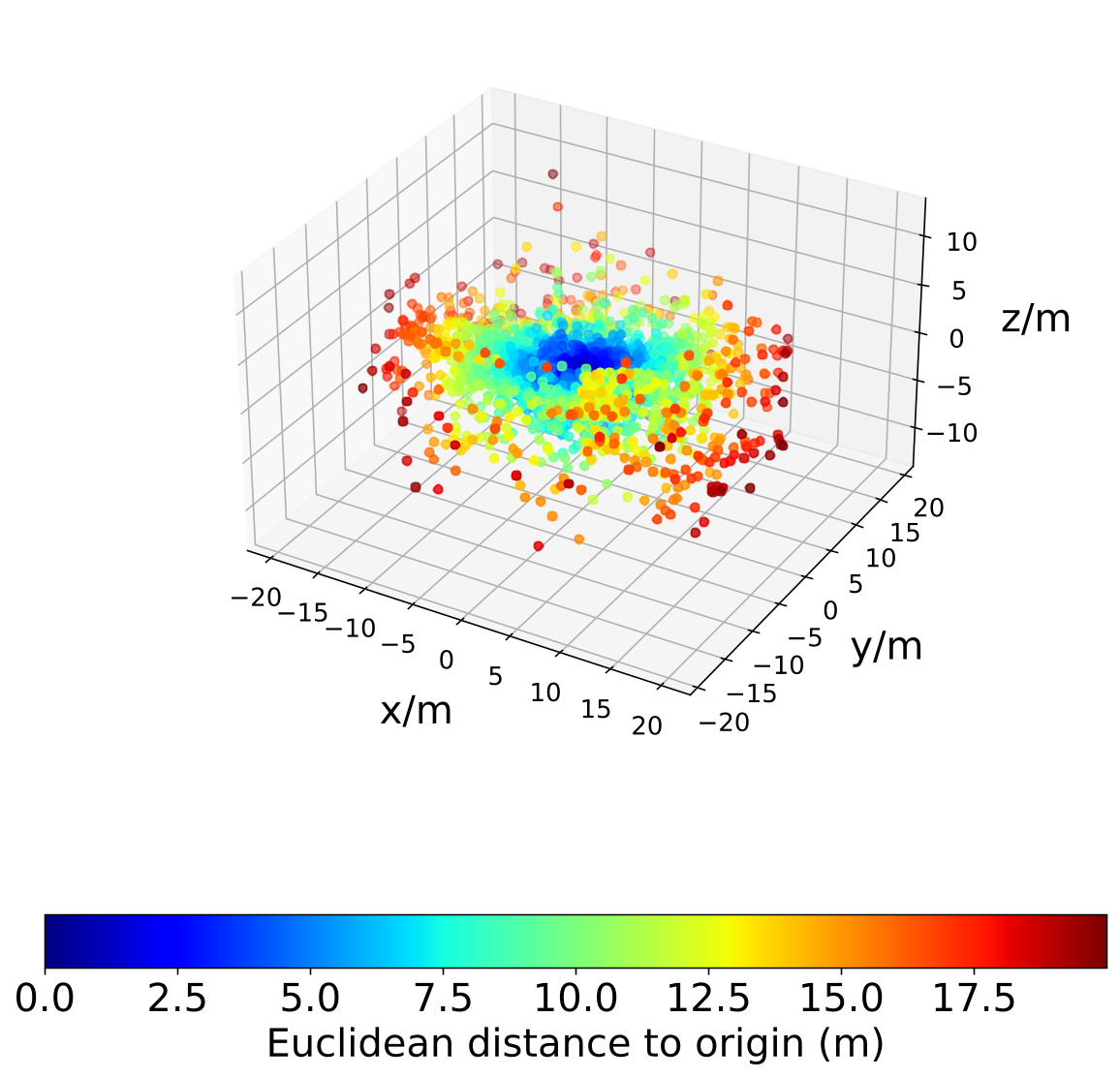}
      \caption{Residual points}
      \label{fig:uavlm_residual}
  \end{subfigure}
  \begin{subfigure}{0.2\linewidth}
      \includegraphics[width=\linewidth]{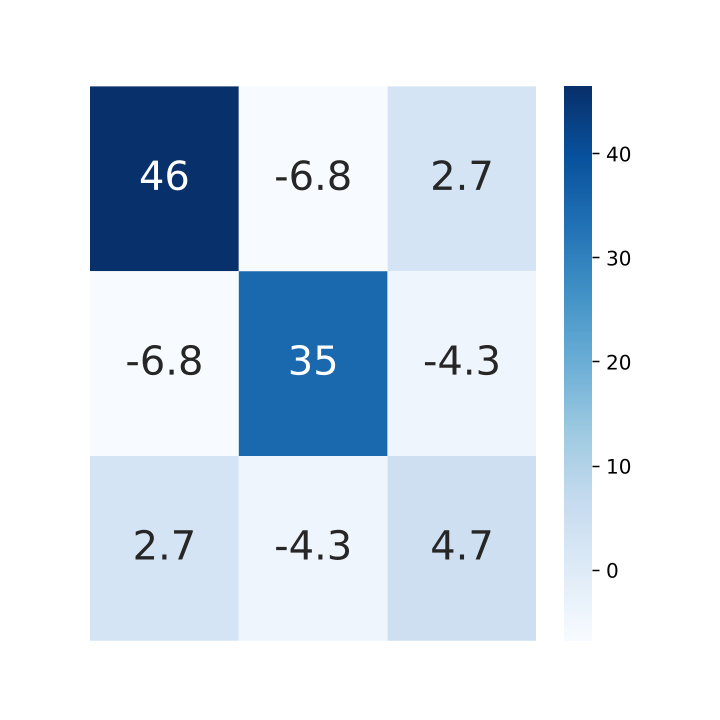}
      \caption{Covariance}
      \label{fig:uavlm_cov_mat}
  \end{subfigure}
  \caption{Introduced vision-based UAV localization is essentially a PnP problem by matching the 3D points in the satellite images and 2D points in UAV recorded images. We only show a few point correspondences out of dozens for demonstration. The satellite image on the left of (a) is provided by Bing Maps and is lifted by DEM to obtain the elevation. The right image of (a) is shot using the UAV's onboard camera. The uncertainty of the data is shown in (b) and (c).}
  \label{fig:uav}
\end{figure*}

Vision-based localization with a monocular camera is a promising navigation technique that is low-cost, low-weight and low-power. It can be complementary when GNSS fails~\cite{8793558, 8794228, 9196606, 9830867}. This section presents a vision-based localization method relying on geo-tagged satellite images and a digital elevation model (DEM), as demonstrated in Fig.~\ref{fig:uav_localization}. We assume the reference satellite image is first retrieved by geo-localization technique~\cite{8794228, 9360609}. By combining the geo-tag of the satellite image and the DEM model, we observe each pixel’s longitude, latitude and elevation, as shown on the left in Fig.~\ref{fig:uav_localization}. Then, cross-domain image registration is applied between the satellite image and the onboard camera recorded image (right in Fig.~\ref{fig:uav_localization}) by utilizing the state-of-the-art feature extractor SuperPoint~\cite{superpoint} and matcher LightGlue~\cite{lindenberger2023lightglue}. Finally, this UAV visual localization problem is essentially concluded as a PnP problem. Note that the precision of the DEM (providing the elevation) and the satellite images' geo-tag (providing the geodetic position) are inconsistent since they are measured by different techniques, resulting in the uncertainty of the observation anisotropic.

\subsubsection{Data Collection and Preparation}
The data was recorded by a fixed-wing UAV's onboard sensors on a flight in Fangshan, Beijing, on January 4th 2024, at 200 to 400 meters above the ground, covering a distance of $27.2$ kilometers. Images were collected by a downward-facing SLR camera mounted strap-down on the belly of the aircraft. 196 images are recorded with resolution $8688\times 5792$ pixels and down-sampled to $1500\times 1000$ pixels. Ground truth poses are obtained by a set of GNSS-INS navigation devices. Reference satellite images are collected using the Bing Maps API, with a ground resolution of around 0.4 meters per pixel. The DEM data is from Copernicus DEM GLO-30 with a 30-meter resolution (4 meters vertical accuracy) and interpolated to align with the satellite image. 

\subsubsection{Results} 
The localization errors by different methods are shown in Table~\ref{tab:uav}, among which GMLPnP achieves the best results, especially for the estimation of the elevation. The overall translation is more accurate than the best baseline PPnP by $29.7\%$, particularly $34.4\%$ in elevation. Some methods fail to localize due to the noisy observation, while our method are robust to provide more accurate results. The uncertainty of the data is also shown in Fig.~\ref{fig:uavlm_residual} and Fig.~\ref{fig:uavlm_cov_mat} with the same setup in section~\ref{seq:ani_uncer}. This experiment presents a method for the localization of UAV which is promising to be the complementary or replacement of onboard GNSS system in the event of a noisy or unreliable GNSS signal.

\begin{table}[tb]
  \caption{Vision-based UAV localization results, reported in absolute errors.}
  \label{tab:uav}
  \centering
  \begin{tabular}{|c|ccc|ccc|}
    \hline
     & \multicolumn{3}{c|}{rotation (degrees)} & \multicolumn{3}{c|}{translation (meters)} \\ 
     & roll & pitch & yaw & east & north & elevation \\ 
    \hline
    EPnP & 13.6 & 17.8 & 20.3 & 8.6 & 9.5 & 28.0\\ 
    BA & 11.6 & 14.0 & 15.8 & \underline{6.7} & 9.1 & 28.3 \\ 
    PPnP & \underline{11.6} & \underline{14.0} & 
    \underline{15.7} & \textbf{6.6} & \underline{9.1} & \underline{27.9} \\ 
    SQPnP & 12.0 & 16.0 & 17.5 & 7.8 & 8.5 & 27.4 \\ 
    UPnP & 33.8 & 35.8 & 40.0 & 16.8 & 17.2 & 93.3 \\ 
    CPnP & 48.1 & 56.4 & 94.6 & 13.7 & 12.9 & 761.5 \\ 
    MLPnP & 95.4 & 94.1 & 29.3 & 11.1 & 13.1 & 120.2 \\ 
    EPnPU & 38.3 & 40.4 & 20.1 & 8.9 & 10.0 & 135.5 \\ 
    DLSU & 25.0 & 33.4 & 40.9 & 14.3 & 19.8 & 52.6 \\ 
    REPPnP & 31.4 & 26.8 & 33.7 & 11.7 & 11.4 & 83.0 \\ 
    GMLPnP(ours) & \textbf{11.0} & \textbf{12.9} & \textbf{14.6} & 7.2 & \textbf{7.8} & \textbf{18.3} \\
  \hline
  \end{tabular}
\end{table}

\section{Conclusion}

In this paper, we propose GMLPnP, a generalized maximum likelihood PnP solver under the discovery of anisotropy uncertainty in many real-world data. The proposed method incorporates the anisotropic uncertainty of feature points into the PnP problem. Experimental results demonstrate increased accuracy in both synthetic and real data. This work is motivated initially by the vision-based localization of UAVs and is meant to explore a more accurate estimation under circumstances where the observations are very noisy. However, our work only considers the case of one central camera and supposes an initial guess is available for pose estimation. Methods for multiple camera systems with both central and non-central cameras are needed for further exploration.

\bibliographystyle{IEEEtran}
\bibliography{ref}

\begin{thebibliography}{10}
\providecommand{\url}[1]{#1}
\csname url@samestyle\endcsname
\providecommand{\newblock}{\relax}
\providecommand{\bibinfo}[2]{#2}
\providecommand{\BIBentrySTDinterwordspacing}{\spaceskip=0pt\relax}
\providecommand{\BIBentryALTinterwordstretchfactor}{4}
\providecommand{\BIBentryALTinterwordspacing}{\spaceskip=\fontdimen2\font plus
\BIBentryALTinterwordstretchfactor\fontdimen3\font minus \fontdimen4\font\relax}
\providecommand{\BIBforeignlanguage}[2]{{%
\expandafter\ifx\csname l@#1\endcsname\relax
\typeout{** WARNING: IEEEtran.bst: No hyphenation pattern has been}%
\typeout{** loaded for the language `#1'. Using the pattern for}%
\typeout{** the default language instead.}%
\else
\language=\csname l@#1\endcsname
\fi
#2}}
\providecommand{\BIBdecl}{\relax}
\BIBdecl

\bibitem{ORBSLAM3_TRO}
C.~Campos, R.~Elvira, J.~J.~G. Rodríguez, J.~M. M.~Montiel, and J.~D.~Tardós, ``Orb-slam3: An accurate open-source library for visual, visual–inertial, and multimap slam,'' \emph{IEEE Transactions on Robotics}, vol.~37, no.~6, pp. 1874--1890, 2021.

\bibitem{vinsmono}
T.~Qin, P.~Li, and S.~Shen, ``Vins-mono: A robust and versatile monocular visual-inertial state estimator,'' \emph{IEEE Transactions on Robotics}, vol.~34, no.~4, pp. 1004--1020, 2018.

\bibitem{schoenberger2016sfm}
J.~L. Sch\"{o}nberger and J.-M. Frahm, ``Structure-from-motion revisited,'' in \emph{Conference on Computer Vision and Pattern Recognition (CVPR)}, 2016.

\bibitem{cpnp}
G.~Zeng, S.~Chen, B.~Mu, G.~Shi, and J.~Wu, ``Cpnp: Consistent pose estimator for perspective-n-point problem with bias elimination,'' in \emph{2023 IEEE International Conference on Robotics and Automation (ICRA)}, 2023, pp. 1940--1946.

\bibitem{epnpu}
A.~Vakhitov, L.~Ferraz, A.~Agudo, and F.~Moreno-Noguer, ``Uncertainty-aware camera pose estimation from points and lines,'' in \emph{Proceedings of the IEEE/CVF Conference on Computer Vision and Pattern Recognition (CVPR)}, June 2021, pp. 4659--4668.

\bibitem{epnp}
V.~Lepetit, F.~Moreno-Noguer, and P.~Fua, ``Epnp: An accurate o(n) solution to the pnp problem,'' \emph{International journal of computer vision}, vol.~81, pp. 155--166, 2009.

\bibitem{sqpnp}
G.~Terzakis and M.~Lourakis, ``A consistently fast and globally optimal solution to the perspective-n-point problem,'' in \emph{Computer Vision -- ECCV 2020}, A.~Vedaldi, H.~Bischof, T.~Brox, and J.-M. Frahm, Eds.\hskip 1em plus 0.5em minus 0.4em\relax Cham: Springer International Publishing, 2020, pp. 478--494.

\bibitem{kneip2011p3p}
L.~Kneip, D.~Scaramuzza, and R.~Siegwart, ``A novel parametrization of the perspective-three-point problem for a direct computation of absolute camera position and orientation,'' in \emph{CVPR 2011}, 2011, pp. 2969--2976.

\bibitem{bujnak2008p4p}
M.~Bujnak, Z.~Kukelova, and T.~Pajdla, ``A general solution to the p4p problem for camera with unknown focal length,'' in \emph{2008 IEEE Conference on Computer Vision and Pattern Recognition}, 2008, pp. 1--8.

\bibitem{dls}
J.~A. Hesch and S.~I. Roumeliotis, ``A direct least-squares (dls) method for pnp,'' in \emph{2011 International Conference on Computer Vision}, 2011, pp. 383--390.

\bibitem{hartley2003multiple}
R.~Hartley and A.~Zisserman, \emph{Multiple view geometry in computer vision}.\hskip 1em plus 0.5em minus 0.4em\relax Cambridge university press, 2003.

\bibitem{mur2015orbslam}
R.~Mur-Artal, J.~M.~M. Montiel, and J.~D. Tardós, ``Orb-slam: A versatile and accurate monocular slam system,'' \emph{IEEE Transactions on Robotics}, vol.~31, no.~5, pp. 1147--1163, 2015.

\bibitem{forstner2010minimal}
W.~F{\"o}rstner, ``Minimal representations for uncertainty and estimation in projective spaces,'' in \emph{Asian Conference on Computer Vision}.\hskip 1em plus 0.5em minus 0.4em\relax Springer, 2010, pp. 619--632.

\bibitem{reppnp}
L.~Ferraz, X.~Binefa, and F.~Moreno-Noguer, ``Very fast solution to the pnp problem with algebraic outlier rejection,'' in \emph{2014 IEEE Conference on Computer Vision and Pattern Recognition}, 2014, pp. 501--508.

\bibitem{ppnp}
V.~Garro, F.~Crosilla, and A.~Fusiello, ``Solving the pnp problem with anisotropic orthogonal procrustes analysis,'' in \emph{2012 Second International Conference on 3D Imaging, Modeling, Processing, Visualization \& Transmission}, 2012, pp. 262--269.

\bibitem{de1994block}
J.~De~Leeuw, ``Block-relaxation algorithms in statistics,'' in \emph{Information Systems and Data Analysis: Prospects—Foundations—Applications}.\hskip 1em plus 0.5em minus 0.4em\relax Springer, 1994, pp. 308--324.

\bibitem{sun2024efficient}
Q.~Sun, T.~Zhang, G.~Zhang, K.~Wang, D.~Zhu, J.~Li, and X.~Zhang, ``Efficient solution to pnp problem based on vision geometry,'' \emph{IEEE Robotics and Automation Letters}, vol.~9, no.~4, pp. 3100--3107, 2024.

\bibitem{urban2016mlpnp}
\BIBentryALTinterwordspacing
S.~Urban, J.~Leitloff, and S.~Hinz, ``Mlpnp - a real-time maximum likelihood solution to the perspective-n-point problem,'' \emph{ISPRS Annals of the Photogrammetry, Remote Sensing and Spatial Information Sciences}, vol. III-3, pp. 131--138, 2016. [Online]. Available: \url{https://isprs-annals.copernicus.org/articles/III-3/131/2016/}
\BIBentrySTDinterwordspacing

\bibitem{ceppnp}
L.~Ferraz~Colomina, X.~Binefa, and F.~Moreno-Noguer, ``Leveraging feature uncertainty in the pnp problem,'' in \emph{Proceedings of the BMVC 2014 British Machine Vision Conference}, 2014, pp. 1--13.

\bibitem{upnp}
L.~Kneip, H.~Li, and Y.~Seo, ``Upnp: An optimal o(n) solution to the absolute pose problem with universal applicability,'' in \emph{Computer Vision -- ECCV 2014}, D.~Fleet, T.~Pajdla, B.~Schiele, and T.~Tuytelaars, Eds.\hskip 1em plus 0.5em minus 0.4em\relax Cham: Springer International Publishing, 2014, pp. 127--142.

\bibitem{shree2001general}
K.~SHREE, ``A general imaging model and a method for finding its parameters,'' in \emph{Proceedings of ICCV, 2001}, 2001.

\bibitem{schweighofer2008globally}
G.~Schweighofer and A.~Pinz, ``Globally optimal o (n) solution to the pnp problem for general camera models.'' in \emph{BMVC}.\hskip 1em plus 0.5em minus 0.4em\relax Citeseer, 2008, pp. 1--10.

\bibitem{zellner1962}
\BIBentryALTinterwordspacing
A.~Zellner, ``An efficient method of estimating seemingly unrelated regressions and tests for aggregation bias,'' \emph{Journal of the American Statistical Association}, vol.~57, no. 298, pp. 348--368, 1962. [Online]. Available: \url{http://www.jstor.org/stable/2281644}
\BIBentrySTDinterwordspacing

\bibitem{malinvaud1980statistical}
E.~Malinvaud, \emph{Statistical Methods of Econometrics}.\hskip 1em plus 0.5em minus 0.4em\relax North-Holland Publishing Company, 1980.

\bibitem{seber2003nonlinear}
G.~A. Seber and C.~J. Wild, ``Nonlinear regression. hoboken,'' \emph{New Jersey: John Wiley \& Sons}, vol.~62, no.~63, p. 1238, 2003.

\bibitem{bates1988nonlinear}
D.~Bates, ``Nonlinear regression analysis and its applications,'' \emph{Wiley Series in Probability and Statistics}, 1988.

\bibitem{Phillips1976THEIM}
P.~C.~B. Phillips, ``The iterated minimum distance estimator and the quasi-maximum likelihood estimator,'' \emph{Econometrica}, vol.~44, pp. 449--460, 1976.

\bibitem{g2o}
R.~Kümmerle, G.~Grisetti, H.~Strasdat, K.~Konolige, and W.~Burgard, ``G2o: A general framework for graph optimization,'' in \emph{2011 IEEE International Conference on Robotics and Automation}, 2011, pp. 3607--3613.

\bibitem{sturm12iros}
J.~Sturm, N.~Engelhard, F.~Endres, W.~Burgard, and D.~Cremers, ``A benchmark for the evaluation of rgb-d slam systems,'' in \emph{Proc. of the International Conference on Intelligent Robot Systems (IROS)}, Oct. 2012.

\bibitem{Liao2022PAMI}
Y.~Liao, J.~Xie, and A.~Geiger, ``{KITTI}-360: A novel dataset and benchmarks for urban scene understanding in 2d and 3d,'' \emph{Pattern Analysis and Machine Intelligence (PAMI)}, 2022.

\bibitem{mei2007single}
C.~Mei and P.~Rives, ``Single view point omnidirectional camera calibration from planar grids,'' in \emph{Proceedings 2007 IEEE International Conference on Robotics and Automation}, 2007, pp. 3945--3950.

\bibitem{8793558}
H.~Goforth and S.~Lucey, ``Gps-denied uav localization using pre-existing satellite imagery,'' in \emph{2019 International Conference on Robotics and Automation (ICRA)}, 2019, pp. 2974--2980.

\bibitem{8794228}
A.~Shetty and G.~X. Gao, ``Uav pose estimation using cross-view geolocalization with satellite imagery,'' in \emph{2019 International Conference on Robotics and Automation (ICRA)}, 2019, pp. 1827--1833.

\bibitem{9196606}
B.~Patel, T.~D. Barfoot, and A.~P. Schoellig, ``Visual localization with google earth images for robust global pose estimation of uavs,'' in \emph{2020 IEEE International Conference on Robotics and Automation (ICRA)}, 2020, pp. 6491--6497.

\bibitem{9830867}
J.~Kinnari, F.~Verdoja, and V.~Kyrki, ``Season-invariant gnss-denied visual localization for uavs,'' \emph{IEEE Robotics and Automation Letters}, vol.~7, no.~4, pp. 10\,232--10\,239, 2022.

\bibitem{9360609}
T.~Wang, Z.~Zheng, C.~Yan, J.~Zhang, Y.~Sun, B.~Zheng, and Y.~Yang, ``Each part matters: Local patterns facilitate cross-view geo-localization,'' \emph{IEEE Transactions on Circuits and Systems for Video Technology}, vol.~32, no.~2, pp. 867--879, 2022.

\bibitem{superpoint}
D.~DeTone, T.~Malisiewicz, and A.~Rabinovich, ``Superpoint: Self-supervised interest point detection and description,'' in \emph{Proceedings of the IEEE Conference on Computer Vision and Pattern Recognition (CVPR) Workshops}, June 2018.

\bibitem{lindenberger2023lightglue}
P.~Lindenberger, P.-E. Sarlin, and M.~Pollefeys, ``{LightGlue: Local Feature Matching at Light Speed},'' in \emph{ICCV}, 2023.

\end{thebibliography}

\end{document}